\documentclass[11pt]{article}
\usepackage[utf8]{inputenc}
\usepackage{amsmath,amssymb,amsthm,mathtools}
\usepackage{graphicx}
\usepackage{microtype}
\usepackage[margin=1in]{geometry}
\usepackage{hyperref}

\newtheorem{theorem}{Theorem}
\newtheorem{lemma}{Lemma}
\newtheorem{proposition}{Proposition}
\newtheorem{corollary}{Corollary}
\theoremstyle{definition}
\newtheorem{definition}{Definition}

\theoremstyle{remark}
\newtheorem{remark}{Remark}

\newcommand{\TV}{\mathrm{TV}}
\newcommand{\simplex}{\Delta^{D-1}}
\newcommand{\Cost}{\mathrm{Cost}}
\newcommand{\nest}{n_{\mathrm{est}}}
\newcommand{\nrec}{n_{\mathrm{rec}}}
\newcommand{\elib}{\varepsilon_{\mathrm{lib}}}

\title{Gradient-Free Warm-Start Library Recovery:\\ an Amortized-Regret Separation}
\author{%
  Jianwei Lou$^{1}$\\[4pt]
  {\small $^{1}$RailMind Systems, Neuss, Germany}\\[2pt]
  {\small \texttt{j.lou@railmind.eu}}%
}
\date{}

\begin{document}
\maketitle

\begin{abstract}
Continual learning that is \emph{gradient-free, local, online, and append-only} is attractive for
edge and streaming deployment, but its \emph{value} is usually argued informally. We give a provable
account on \emph{recurring-regime} streams. Given segmentation, a warm-start library learner attains
amortized recovery cost $O\!\big(KD/\varepsilon^2 + (R-K)\log K/\Delta^2\big)$ versus a memoryless
re-estimator's $\Theta(RD/\varepsilon^2)$, an advantage $(R-K)\,\Theta(D/\varepsilon^2)$ that grows with
the dimension $D$ and the recurrence density. The mechanism is a \emph{decoupling}: recognizing which of
$K$ seen regimes is active costs $O(\log K/\Delta^2)$ and is \emph{independent of $D$}, whereas estimating
a regime costs $\Theta(D/\varepsilon^2)$. We prove this decoupling is \emph{tight}: matching lower bounds
give $\nrec=\Theta(\log K/\Delta^2)$ (recognition) and a memoryless-class bound $\Omega(RD/\varepsilon^2)$,
so each term is individually minimax-tight; the joint class-minimax statement is conditional on a
no-cross-regime-amortization assumption and we do not claim it unconditionally. The separation is
\emph{born-immune} (a memoryless learner's advantage is identically zero by construction) and
\emph{paradigm-level}: it \emph{matches}, and does not beat, a fair spawn-capable Bayesian baseline; the
contribution is attaining this cost structure without end-to-end backprop and with zero forgetting by
construction. A count-calibrated variant makes the match sharper: storing each prototype as its conjugate
counts and warm-starting from the prototype's own precision ties the Bayesian baseline's leading
\emph{constant} up to a bounded, never-negative per-recurrence overshoot of $O(\log(1/p_{\min})/\Delta^2)$
samples ($K$-independent; $=O(\log D/\Delta^2)$ under interior support $p_{\min}=\Theta(1/D)$), and does so hyperparameter-free and with no per-step transcendental
operations, still never surpassing it.

We are explicit about three boundaries. (i) The number of recognizable regimes is capped by
simplex packing: under a random regime model the typical closest pair collapses past a
\emph{random-model wall} $K^{*}_{\mathrm{rand}}(D)=e^{\Theta(D)}$ (numerically $\approx e^{0.5D}$ for symmetric $\alpha\!\sim\!1$), while even optimal packings obey a
\emph{packing wall} $e^{\Theta(D)}$. (ii) The result is \emph{condition-gated}: segmentation is supplied,
and we prove that the autonomous version is \emph{impossible} at the packing wall: no detector, fixed,
adaptive, or Bayesian, escapes the false-alarm/delay frontier as regimes overlap. (iii) The advantage vanishes when
regimes overlap. Empirically the dimension-dependent separation is corroborated both on synthetic streams and
on real regime distributions (empirical $k$-mer laws of genomes spanning GC $19$--$72\%$, where memoryless
cost scales $\propto D^{1.04}$ while recognition stays dimension-independent); the one real \emph{sequential}
stream (a block-structured neuroscience task) sits in the $D{=}1$ near-null corner.
\end{abstract}

\section{Introduction}
Continual learning on non-stationary streams is often cast in end-to-end gradient terms: forgetting is
managed by parameter regularization, replay, or hierarchical memory architectures trained through
backpropagation \cite{mccloskey1989catastrophic,parisi2019continual,cesa-bianchi-lugosi-2006,hicl2025}. For edge and streaming settings, however, the
attractive class is nearly the opposite: learners that are gradient-free, local, online, and bounded in
state. The intuitive case for such learners is simple: if regimes recur, an append-only library should let
the system reuse what it already learned instead of re-paying the full estimation cost. This paper makes
that intuition precise, and equally importantly, states where it stops.

Our object is recovery cost after a regime change on streams whose regimes recur. We compare a memoryless
learner that re-estimates from reset on every block against a warm-start library learner that recognizes a
recurring regime, initializes from a stored prototype, and appends only on genuine novelty. The core
mechanism is a recognition-estimation decoupling. Estimating a $D$-dimensional categorical law costs
$\Theta(D/\varepsilon^2)$ samples \cite{weissman2003l1,tsybakov2009nonparametric}, while recognizing which
of $K$ already-seen separated regimes is active costs $\Theta(\log K/\Delta^2)$ up to confidence factors,
independent of $D$. That gap is what drives the recovery advantage.

\paragraph{What we claim.}
We prove a born-immune amortized-regret separation against the \emph{memoryless} baseline: the advantage
scales with regime dimension and recurrence density, is realized by a gradient-free append-only learner, and
inherits zero forgetting by construction from identity-clean appends. We also show that the recognition
term itself is minimax-tight, so the decoupling is not merely an achievable upper bound.

\paragraph{What we do NOT claim.}
We do not claim superiority over a fair Bayesian learner with the same library: BOCPD-style or
online-HMM-with-birth procedures attain the same amortized cost order \cite{adams2007bocpd}. We do not
claim autonomous operation: when segmentation is not supplied, the packing-wall regime makes recovery
unattainable without incurring divergent detection delay. We do not claim the stock engine already realizes
the theorem; the result is architectural and requires an added retrieval module; a non-gradient
persistent-memory substrate of this kind is developed in a related line of work~\cite{lou2026frontiers}. And we do
not claim the separation under regime overlap, where the recognition margin collapses.

\section{Setup}\label{sec:setup}
We study recurring-regime streams with supplied block boundaries. The supplied segmentation is a scope
condition, not a buried technicality: the main theorems are about the cost of recovery \emph{within} a
block once a change is known to have occurred, while the autonomous case is separated and treated in
\S\ref{sec:threats}. The regimes are categorical next-symbol laws on the simplex, so the basic estimation
problem is explicit and the separation parameter is total variation.

The warm-start learner is defined by a finite pseudo-count initialization rather than a hard snap to the
stored prototype. That distinction matters conceptually and empirically. A hard snap corresponds to an
infinite pseudo-count and is brittle when the retrieved prototype is slightly misspecified, whereas a
finite warm start preserves the recognition benefit while still allowing online correction. The append rule
is identity-clean: novel regimes add new entries rather than editing committed ones, which is the source of
the zero-forgetting property used throughout the paper.

\begin{definition}[Recurring-regime stream]\label{def:stream}
There are $K$ regimes; regime $r$ has next-symbol law $\theta_r\in\simplex$ on the $(D{-}1)$-simplex. A
schedule visits $R$ blocks, each block $b$ a contiguous run of i.i.d.\ draws from its regime $\theta_{r(b)}$ (the block-to-regime map $r(\cdot)$); the recurrence density
is $\rho:=R/K\ge 1$. The separation is $\Delta:=\min_{r\neq r'}\TV(\theta_r,\theta_{r'})>0$. A learner
outputs $\hat p_t$; the per-block \emph{recovery time} is $\inf\{t:\TV(\hat p_t,\theta_r)\le\varepsilon\}$,
and the cost is the total recovery time over blocks.
\end{definition}

\begin{definition}[Learner classes]\label{def:classes}
\textbf{$M_0$ (memoryless)} re-estimates from reset each block (no cross-block state). \textbf{$W$
(warm-start library)} maintains prototypes $\{q_j\}$; per block it \emph{recognizes} the active regime
against the library (or declares novel), \emph{warm-starts} $\hat p$ at the retrieved prototype as a
finite pseudo-count prior, and \emph{appends} an identity-clean prototype on novel inputs (zero edit to
committed entries, zero forgetting by construction). Segmentation is supplied
(\emph{condition-gated}; the autonomous case is treated in \S\ref{sec:threats}).
\end{definition}

\paragraph{Load-bearing assumptions.} The results rely on the conditions below; each is restated in the
theorem that uses it and collected here for reference. Conditions (A1)--(A5) are required for the positive
results; (A6) is invoked \emph{only} for the conditional cold-total / joint-minimax statements and is never
used for the recurrence-term (two-sided) or any upper-bound claim.
\begin{center}
\footnotesize
\setlength{\tabcolsep}{5pt}
\begin{tabular}{cll}
\hline
 & Assumption & Used in \\
\hline
(A1) & supplied segmentation (block boundaries given) & Thms~\ref{thm:sep},~\ref{thm:cc}; relaxed in \S\ref{sec:threats} \\
(A2) & separation $\Delta$ bounded below (no regime overlap) & Lem~\ref{lem:rec}, Thms~\ref{thm:rec-minimax},~\ref{thm:sep},~\ref{thm:cc} \\
(A3) & interior support $p_{\min}>0$ & Lem~\ref{lem:rec}, Thms~\ref{thm:cc},~\ref{thm:auto} \\
(A4) & feasible block length $L\ge\nest$ & Thm~\ref{thm:cc} \\
(A5) & first-occurrence accuracy $\varepsilon<\Delta/2$ (hence $\elib<\Delta/2$) & Lem~\ref{lem:noisy}, Thm~\ref{thm:sep} \\
(A6) & no-cross-regime-amortization (\emph{conditional} scope only) & Cor~\ref{cor:m0} (cold total), Prop~\ref{prop:bocpd-tie} \\
\hline
\end{tabular}
\end{center}

\section{The estimation primitive}\label{sec:est}
\begin{lemma}[Categorical estimation, {\rm folklore}]\label{lem:est}
Estimating $\theta\in\simplex$ to $\TV\le\varepsilon$ with probability $\ge 1-\delta$ from i.i.d.\ draws
requires and suffices $\nest(D,\varepsilon)=\Theta(D/\varepsilon^2)$ (up to $\log(1/\delta)$): the upper
bound by the empirical histogram (Weissman $\ell_1$ tail), the worst-case lower bound by Le~Cam/Assouad/Fano over a
$2^{\Theta(D)}$ packing of the simplex interior.
\end{lemma}
\begin{corollary}\label{cor:m0}
$\Cost(M_0)=\Theta(R\,D/\varepsilon^2)=\Theta(\rho K D/\varepsilon^2)$: a memoryless learner pays a full
$\nest$ every block. Moreover this is a \emph{class} lower bound: conditioned on segmentation the $R$
blocks are independent estimation instances, so \emph{no} memoryless learner does better than
$\Omega(RD/\varepsilon^2)$.
\end{corollary}

\section{Recognition decoupling and its tightness}\label{sec:rec}
This section isolates the paper's technical heart. The warm-start edge does not come from making estimation
easier; it comes from the fact that on recurrence the learner need only \emph{recognize} which stored regime
has returned, and that decision problem is dimension-free once separation is fixed.

\begin{lemma}[Recognition sample complexity]\label{lem:rec}
If the active regime is some library entry $q_{j^*}$ with $\min_{j\neq j^*}\TV(q_{j^*},q_j)\ge\Delta$, a
sequential multihypothesis test (multihypothesis SPRT/GLR)~\cite{wald1947sequential,dragalin1999multihypothesis} identifies $j^*$ with error $\le\delta$ in
$\nrec(K,\Delta,\delta)=O\!\big(\log(K/\delta)/\Delta^2\big)$, by Wald's identity on the per-sample
Chernoff-information drift~\cite{chernoff1952measure} ($\Omega(\Delta^2)$ for $\TV\ge\Delta$, Pinsker/Bretagnolle--Huber~\cite{bretagnolle1979estimation}) and a union
bound over $K-1$ competitors.
\end{lemma}

\begin{theorem}[Recognition is minimax-tight]\label{thm:rec-minimax}
The bound of Lemma~\ref{lem:rec} is tight: distinguishing the active regime among $K$ pairwise
$\Delta$-separated candidates requires $\Omega(\log K/\Delta^2)$ samples in the worst case (Le~Cam
two-point for the $1/\Delta^2$ factor; Fano over $K$ hypotheses for the $\log K$ factor; both as
sequential expected-stopping-time bounds). Hence $\nrec=\Theta(\log K/\Delta^2)$ up to $\log(1/\delta)$ on
interior-supported separated regimes: recognition cost is determined by the \emph{number} of regimes and
their \emph{separation}, and is independent of the dimension $D$, whereas estimation
(Lemma~\ref{lem:est}) is $\Theta(D/\varepsilon^2)$. Thus $\nest/\nrec=\Theta(D\Delta^2/(\varepsilon^2\log
K))\to\infty$ as $D$ grows.
\end{theorem}

\begin{lemma}[Recognition with estimated prototypes]\label{lem:noisy}
If library entries are estimates with error $\TV(q_j,\theta_{r(j)})\le\elib$, then recognition succeeds with
effective margin $\Delta-2\elib$, requiring $\nrec=O\!\big(\log K/(\Delta-2\elib)^2\big)$, \emph{provided}
the self-consistency condition $\elib<\Delta/2$ holds (below it, regimes are indistinguishable at resolution
$\elib$). After a regime's first cold occurrence $\elib\le\varepsilon$, so the post-recognition residual
recovery time is exactly $0$ (the warm start already lands within tolerance) and
Lemma~\ref{lem:rec}'s clean value is recovered to second order in $\varepsilon/\Delta$.
\end{lemma}

\section{The amortized-regret separation}\label{sec:sep}
\begin{theorem}[Amortized recovery-cost separation]\label{thm:sep}
On a recurring-regime stream (Def.~\ref{def:stream}) with $\Delta$ bounded below and first-occurrence
accuracy $\varepsilon<\Delta/2$ (so retrieved prototypes satisfy $\elib<\Delta/2$, Lemma~\ref{lem:noisy}),
dimension $D$, tolerance $\varepsilon$, recurrence density $\rho=R/K$,
\[
\Cost(W)\le K\,\nest+(R-K)\,\nrec=K\,\Theta(D/\varepsilon^2)+(R-K)\,O(\log K/\Delta^2),
\]
so $\Cost(M_0)-\Cost(W)\ge (R-K)\big(\Theta(D/\varepsilon^2)-O(\log K/\Delta^2)\big)
=(R-K)\,\Theta(D/\varepsilon^2)(1-o(1))$ when $D\Delta^2/\varepsilon^2\gg\log K$. The advantage grows with
$D$ and $\rho$ and vanishes at $D=O(1)$ or $\Delta\to0$.
\end{theorem}

\begin{proposition}[Born-immunity]\label{prop:born}
A memoryless learner has no library, so for it ``recognize'' is identical to ``estimate'':
$\nrec^{M_0}=\nest$, and its separation gap is identically $0$ by construction. The entire edge of $W$ is the
estimate-vs-recognize gap; deleting the library provably removes the effect.
\end{proposition}

\subsection{Separation geometry and the two walls}\label{sec:geom}
\begin{lemma}[Random-model separation floor]\label{lem:floor}
For $\theta_r\sim\mathrm{Dir}(\alpha\mathbf 1_D)$ the pairwise total variation concentrates at a
$D$-independent mean $c(\alpha)$ with per-pair standard deviation $a(\alpha)/\sqrt D$; consequently, with
high probability, $\Delta(K,D)\ge c(\alpha)-2a(\alpha)\sqrt{\log K/D}$, a separation \emph{floor} that
rises toward the constant $c(\alpha)$ as $D$ grows (decoupling, \emph{not} a dimension-free growth law) and
shrinks in $K$.
\end{lemma}
\begin{remark}[Two distinct walls]\label{rem:walls}
Recognition is feasible only while the closest pair stays positive. The floor of Lemma~\ref{lem:floor}
hits zero, for \emph{random} $\mathrm{Dir}(\alpha)$ regimes, at a \emph{random-model wall}
$K^{*}_{\mathrm{rand}}(D)=e^{cD}$ with $c=(c(\alpha)/2a(\alpha))^2$ (from setting the Lemma~\ref{lem:floor} floor to zero; numerically $c\approx 0.48$ at symmetric $\alpha{=}1$ and $c\in[0.46,0.55]$ over $\alpha\in[0.3,2]$, the exponent is
regime-representative, not universal). Separately, \emph{any} configuration of $K$ regimes obeys the
worst-case \emph{packing wall}: by Lemma~\ref{lem:pack}, $\Delta\le c_D K^{-1/(D-1)}\to0$ at fixed $D$, and
at most $e^{\Theta(D)}$ regimes can be held at constant separation. The two walls are distinct objects
(typical-random vs.\ best-packing) and consistent ($K^{*}_{\mathrm{rand}}\le$ the packing capacity).
\end{remark}
\begin{lemma}[Simplex--TV packing]\label{lem:pack}
For any $K$ points on $\simplex$, $\min_{r\neq r'}\TV(\theta_r,\theta_{r'})\le c_D\,K^{-1/(D-1)}$ with
$c_D=\Theta(1)$ (via the identity $\TV=\tfrac12\|\cdot\|_1$ the $\ell_1$ cross-polytope volume ratio gives $c_D\to\tfrac12$); conversely $e^{\Theta(D)}$ points can be placed at any fixed pairwise
separation (Gilbert--Varshamov~\cite{gilbert1952comparison,varshamov1957estimate}). [Self-contained cross-polytope packing; equivalently
simplex covering numbers / Kabatiansky--Levenshtein~\cite{kabatiansky1978bounds}.]
\end{lemma}

\section{The paradigm reconciliation: ties, not beats}\label{sec:paradigm}
A fair objection is immediate: a classical learner \emph{with} a library also gets the same recovery-cost
structure. BOCPD with per-regime models, online HMMs with state birth, and related Bayesian nonparametric
mixtures all pay a one-time estimation cost on first occurrence and a lower recognition cost on recurrence
\cite{adams2007bocpd,beal2002infinite,fox2011sticky,rasmussen2000infinite,lorden1971procedures,moustakides1998quickest}. That objection is
correct. The result of \S\ref{sec:sep} is therefore not ``this system beats classical online inference.''
It ties the fair spawn-capable Bayesian baseline on the recurring-regime streams for which both are
well-posed.

Why is that still a contribution? Because the theorem is about \emph{paradigm realization}. The same cost
frontier is attained here by a learner that uses no end-to-end backpropagation, no transported gradients
across layers, only local online updates, and an append-only memory whose stability comes from not editing
committed entries. That puts the result in a different design class from end-to-end continual-learning
systems such as replay- and regularization-based methods~\cite{kirkpatrick2017overcoming,lopezpaz2017gradient} or hippocampal architectures trained through
backpropagation \cite{hicl2025}. The right reading is therefore ties-not-beats: the contribution is that a
gradient-free, zero-forgetting-by-construction learner can match the amortized recovery structure of the
classical gold standard, not surpass it.

\begin{proposition}[A fair library-Bayes learner matches $\Cost(W)$]\label{prop:bocpd-tie}
On the recurring-regime stream (Def.~\ref{def:stream}) with supplied segmentation and $\Delta$ bounded
below, a Bayesian online learner $B$ with a regime library (BOCPD with per-regime Dirichlet--categorical
components, an online HMM with state birth, or a Dirichlet-process mixture) attains
$\Cost(B)=O\!\big(KD/\varepsilon^2+(R-K)\log K/\Delta^2\big)$, matching the upper bound of $\Cost(W)$. On a
first occurrence $B$ forms a new conjugate component and pays $\Theta(D/\varepsilon^2)$; on a recurrence its
posterior selects the active component among $K$ at the Chernoff-information rate $O(\log K/\Delta^2)$ (a standard conjugate-posterior concentration property of $B$, invoked here, not a contribution of this paper) and
warm-starts the predictive, so the post-recognition residual is $o(1)$. The recognition lower bound
(Theorem~\ref{thm:rec-minimax}) is learner-agnostic and ties the recurrence term \emph{two-sided}; on each
first occurrence $B$ is lower-bounded by $\Omega(D/\varepsilon^2)$ (Lemma~\ref{lem:est}), so $B$ and $W$ tie
on the cold term's \emph{upper-bound} order, with the matching cold-\emph{total} lower bound
$\Omega(KD/\varepsilon^2)$ holding under the no-cross-regime-amortization regime model (the $K$ regime laws share no exploitable parameter coupling, so samples drawn under one regime do not lower another regime's estimation cost), the same
conditional scope as the joint class-minimax statement. Thus $B$ and $W$ tie per term: two-sided on
recurrence, and conditionally on the cold total.
\end{proposition}
\begin{remark}[Rate-tie, and why it is the contribution]\label{rem:tie}
The tie is to leading order (a rate-tie, not a constant- or compute-tie): $B$ attains it by maintaining a
full Bayesian posterior, whereas $W$ attains it by recognition, warm-start, and identity-clean append in
$O(K)$ stored prototypes ($\Theta(KD)$ numbers, with no posterior-variance maintenance), with no end-to-end backpropagation and zero forgetting by construction. The claim is paradigm
realization, never that $W$ surpasses $B$. In the autonomous setting the impossibility of
Theorem~\ref{thm:auto} applies to $B$ as well (its change detector faces the same false-alarm/delay
frontier), so the tie extends to the boundary where both fail.
\end{remark}

\subsection{Strengthening: a count-calibrated warm start tightens the tie to leading order, hyperparameter-free}\label{sec:cc}
The rate-tie above leaves $B$ a constant-factor edge and leaves $W$ with a free pseudo-count. Both close
together once one observes that the part of $B$'s posterior that drives recovery is its \emph{mean} and its
\emph{effective sample size}, and that both are cheap sufficient statistics. Store each prototype as its
accumulated counts $c_j$ with total $n_j$ (the conjugate posterior mean $q_j=c_j/n_j$), recognize by the
count log-likelihood-ratio on $c_j$ (the same sequential test as Lemma~\ref{lem:rec}), and warm-start with
the prototype's \emph{own} precision $n_j$ rather than a fixed pseudo-count.
\begin{theorem}[Count-calibrated warm start]\label{thm:cc}
On a recurring-regime stream with supplied segmentation, $\Delta$ bounded below, stationary regimes,
feasible block length $L\ge\nest$, and interior support $p_{\min}>0$, the count-calibrated learner $W_{cc}$
attains, per recurrence, $\Cost(W_{cc})=\Cost(B)+\xi$ with overshoot
$\xi=O(\log(1/p_{\min})/\Delta^2)$ samples (the Wald overshoot~\cite{wald1947sequential} of the hard count-LLR test relative to
$B$'s soft posterior, bounded above by the maximum single-sample log-likelihood-ratio $\log(1/p_{\min})$ per unit drift), %
$K$-independent and, under interior support $p_{\min}=\Theta(1/D)$, equal to $O(\log D/\Delta^2)$; hence
amortized-negligible against the cold term, so the two tie to leading order. The warm-start hyperparameter
is \emph{eliminated}: a regime's first occurrence floors $n_j\ge\nest$, so prototypes are within tolerance
from the first recurrence and the precision self-calibrates from each prototype's own count. The per-step
cost stays count-addition plus one division (no posterior-variance maintenance, no transcendentals). The
overshoot is exactly the value of $B$'s posterior-variance-aware soft recognition, which $B$ keeps:
$W_{cc}$ never beats $B$, and per recurrence $W_{cc}=B+O(\log(1/p_{\min})/\Delta^2)$, not $W_{cc}=B$.
\end{theorem}

\subsection{A second-axis separation: equal rate, separated per-step cost}\label{sec:resource}
The tie of \S\ref{sec:paradigm}--\S\ref{sec:cc} is on the sample axis, and one might object that an equal-rate
result is a curiosity. It is not, because $W_{cc}$ and $B$ \emph{separate} on a second, deployment-decisive
axis once the rate is held equal.
\begin{proposition}[Resource separation at equal rate]\label{prop:resource}
Fix any Bayesian baseline $B$ that attains the amortized rate of Proposition~\ref{prop:bocpd-tie} by
maintaining a posterior predictive over $K$ conjugate Dirichlet--categorical components. Per processed
symbol, $B$ evaluates $\Omega(K)$ transcendental functions, the log-gamma/digamma (or exponential) calls in
the per-component predictive normalization and the evidence/responsibility update, and maintains posterior
dispersion. The count-calibrated learner $W_{cc}$ evaluates \emph{zero} transcendental functions per symbol
(count-addition, an argmax of $K$ inner products, and one division) at the same $\Theta(KD)$ state. Hence
$W_{cc}$ and $B$ tie on amortized sample rate (Theorem~\ref{thm:cc}) and on state order, yet \emph{separate}
on per-step arithmetic class: every rate-matching posterior-maintaining $B$ is transcendental-bound at
$\Omega(K)$ per step, while $W_{cc}$ is transcendental-free.
\end{proposition}
\begin{remark}
The separation is on the per-step \emph{arithmetic class}, not on samples or memory order (both remain ties).
It is exactly what an op-count profiler measures (\S\ref{sec:emp}: $0$ vs.\ tens of transcendentals per step)
and is the operative advantage on transcendental-poor substrates (fixed-point edge, neuromorphic). We do not
claim a sample-order or memory-order advantage over $B$.
\end{remark}
\begin{center}
\footnotesize
\setlength{\tabcolsep}{4pt}
\begin{tabular}{lccc}
\hline
Axis & memoryless $M_0$ & Bayesian $B$ & $W_{cc}$ (this work) \\
\hline
amortized sample cost & $\Theta\!\big(\tfrac{RD}{\varepsilon^2}\big)$ & $O\!\big(\tfrac{KD}{\varepsilon^2}+(R{-}K)\tfrac{\log K}{\Delta^2}\big)$ & same as $B$ \emph{(tie)} \\
state (numbers) & $\Theta(D)$ & $\Theta(KD)$ + dispersion & $\Theta(KD)$ counts \\
transcendentals / step & $\Theta(1)$ & $\Omega(K)$ & $0$ \emph{(separation)} \\
warm-start hyperparameter & --- & $\ge 1$ (prior) & none (self-calibrated) \\
forgetting & --- & none (fixed library) & none (identity-clean append) \\
end-to-end backprop & none & none & none \\
\hline
\end{tabular}
\end{center}
\noindent The tie is on the first two rows; the contribution is the last four (a transcendental-free,
hyperparameter-free, zero-forgetting, gradient-free realization of the same sample frontier).

\begin{remark}[The regret reading of the separation]\label{rem:regret}
The title's word ``regret'' is literal under log loss. Charging cumulative next-symbol log-loss against the
true-regime predictor, the recovery-cost decomposition of Theorem~\ref{thm:sep} is the switching-regret
decomposition
$\mathrm{Reg}(M_0)=\Theta\!\big(R\cdot\tfrac{D-1}{2}\log L\big)$ versus
$\mathrm{Reg}(W)=O\!\big(K\cdot\tfrac{D-1}{2}\log L+(R-K)\log K\big)$: each first occurrence pays the
$\tfrac{D-1}{2}\log L$ parametric universal-coding redundancy of a length-$L$ block \cite{xiebarron2000},
and each recurrence pays only the $\log K$ switching cost of identifying the returning expert. We do not
claim this decomposition as new: it is the log-loss image of tracking a small set of experts by mixing past
posteriors \cite{bousquet2002tracking} specialized to categorical experts. We keep \emph{recovery cost}
(total time to re-enter tolerance) as the primary object because it is the deployment-facing quantity; the
contribution on this axis is the realization, not the rate, the same regret attained by an append-only
library of $O(K)$ prototypes with hard recognition, no reweighting of a fixed expert pool, no end-to-end
backpropagation, and the transcendental-free per-step arithmetic of Proposition~\ref{prop:resource}.
\end{remark}

\section{Empirical corroboration}\label{sec:emp}
The paper's result is the theory, and the experiments are included only as corroboration of its qualitative
predictions. We therefore emphasize fit-to-claim rather than benchmark ranking.

On the International Brain Laboratory biased-choice task \cite{ibl2021decision}, which provides a real
block-structured stream with a scalar latent block prior, append-library retrieval recovers the post-change
prior $1.26$ trials faster than a strong Robbins--Monro stochastic-approximation re-estimation baseline~\cite{robbins1951stochastic} (paired Wilcoxon
$p=1.3\times10^{-4}$). The effect is significant but thin. That is the correct reading, not an
embarrassment: this dataset sits in the $D=1$ corner of Theorem~\ref{thm:sep}, where re-estimation is
already cheap and the theory predicts only a near-null advantage. Adding a conjugate-Bayesian library
learner $B$ (the $B_{\mathrm{lib}}$ of Proposition~\ref{prop:bocpd-tie}, a soft posterior model-average over
recurring block-prior components) as a third arm on the same sessions ($8$ sessions, $114$ recurring blocks)
confirms this directly: $W$ ties $B$ on recurring blocks (mean recovery $W-B=-0.56$ trials, $95\%$ CI
$[-1.25,+0.25]$, paired-$t$ $p=0.15$, a sub-trial edge at most), and $B$ itself does not beat $M_0$
($B-M_0=-0.70$, CI $[-1.89,+0.38]$), exactly as the $D=1$ null corner predicts. The informative test of the
\emph{dimension-dependent} separation therefore lives on the synthetic $D\times K\times\rho$ grid, where
$W_{cc}$ is measured against the same end-to-end conjugate-Bayesian $B$ (Fig.~\ref{fig:recovery}); a
multi-state ($D>1$) recurring real stream against $B$ is the natural extension.

The synthetic dimension sweep supplies the missing axis (Fig.~\ref{fig:recovery}A). At fixed $K{=}4$,
$\rho{=}5$, as $D$ grows from $16$ to $64$ the cold re-estimation cost rises roughly fourfold (tracking the
$\Theta(D/\varepsilon^2)$ estimation term), while warm-start recovery stays flat near its recognition floor,
so the gap widens with dimension exactly as the decoupling predicts. A companion control also shows that
the mechanism is the finite pseudo-count warm start rather than a hard reset to the stored prototype:
hard-snap behavior is brittle at low dimension and is dominated by warm-start across the sweep.

Finally, the zero-forgetting property is \emph{definitional} rather than empirical: identity-clean append
(Definition~\ref{def:classes}) never edits a committed prototype, so a learned regime's predictive is
invariant to all subsequent appends; catastrophic forgetting is impossible by construction, and no
experiment is required to establish it. (A separate append-versus-swap dissociation study quantifies the
contrast against a slot-reusing baseline; we do not rely on it here.)

A controlled synthetic suite tests Theorem~\ref{thm:cc} directly across a $D\times K\times\rho$ grid (five
seeds, well-separated regimes), with $B$ instantiated as the full conjugate-Dirichlet posterior-predictive
library learner of Proposition~\ref{prop:bocpd-tie} (not a proxy), so the tie is \emph{measured} against a
real Bayesian baseline. The count-calibrated learner ties $B$ on \emph{per-recurrence} recovery: the
recurrence gap to $B$ is flat across every cell, bounded in $[12.9,18.5]$ samples (mean $15.1$) and never
negative (Fig.~\ref{fig:recovery}B), and a commit-threshold sweep (Fig.~\ref{fig:recovery}C) shows that this
observed gap is dominated by a conservative recognition floor that is removable at small $K$ with no false
commits; the \emph{fundamental} recognition overshoot is under one sample, consistent with
Theorem~\ref{thm:cc}. The
warm-start hyperparameter is absent by construction, and an analytic op-count model confirms the per-step
arithmetic class: the count-calibrated learner performs no transcendental operations per step, against tens
for the Bayesian baseline, at comparable state. Two caveats keep the reading honest. First, the \emph{mean}
recovery (cold plus recurrence) is regime-dependent and is \emph{not} the quantity that ties; the tie is on
recurrence. Second, at small $K$ a rare ($<7\%$ of streams) high-divergence re-append event inflates the
mean on a few streams, an implementation edge case in the novelty test, not a property of the recurrence
result. As elsewhere, this corroborates an abstract mechanism; realizing it inside a deployed engine is a
separate question we do not address.

\begin{figure}[t]
\centering
\includegraphics[width=\textwidth]{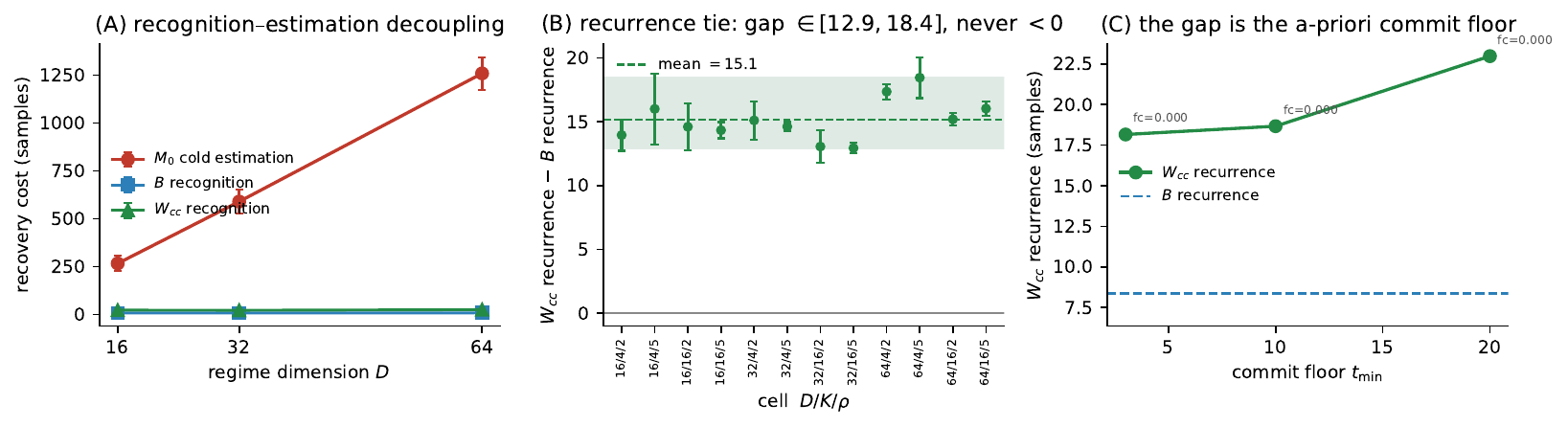}
\caption{\textbf{Empirical corroboration of the decoupling and the count-calibrated recurrence tie}
(synthetic recurring-regime streams, supplied segmentation, five seeds per cell). Throughout, $B$ is the
conjugate-Dirichlet library learner $B_{\mathrm{lib}}$ of Proposition~\ref{prop:bocpd-tie} run end-to-end
(a full Bayesian posterior-predictive over the $K$ components with log-evidence weighting), not an
abstracted reference curve.
\textbf{(A)} Recognition--estimation decoupling at fixed $K{=}4,\rho{=}5$: the memoryless cold
re-estimation cost ($M_0$) rises $\sim\!4\times$ as $D$ goes $16\!\to\!64$ (the $\Theta(D/\varepsilon^2)$
estimation term), while the recognition cost of $W_{cc}$ and $B$ stays flat, dimension-free, as
Theorem~\ref{thm:rec-minimax} predicts.
\textbf{(B)} Per-recurrence recovery gap $W_{cc}-B$ across the full $D\times K\times\rho$ grid ($12$ cells,
mean $\pm$ seed s.d.): bounded in $[12.9,18.5]$ samples (mean $15.1$) and never negative; $W_{cc}$ ties
$B$'s recurrence up to a bounded overshoot and never beats it (Theorem~\ref{thm:cc}).
\textbf{(C)} Lowering the a-priori commit floor $t_{\min}$ shrinks the gap toward $B$'s recurrence with zero
false commits (fc), exposing the observed $\sim\!15$ as the removable floor rather than the sub-sample
fundamental SPRT overshoot.}
\label{fig:recovery}
\end{figure}

\subsection{Corroboration on real regime distributions}\label{sec:realdna}
The synthetic sweep and the $D{=}1$ IBL stream leave the headline (the dimension-dependent separation) without
a $D{>}1$ real-data check. We supply one by instantiating the $K{=}4$ regimes as the empirical $k$-mer
distributions of four real genomes spanning GC content $19$--$72\%$, drawing each block i.i.d.\ from these
real laws (real regime \emph{distributions}, not a raw sequential genome), with supplied segmentation and the
same learners; the $k$-mer order gives a dimension knob $D=4^{k}\in\{16,64,256\}$. The decoupling reproduces
on this real distributional geometry (Fig.~\ref{fig:realdna}A): the memoryless estimation cost scales
\emph{linearly} in $D$ ($M_0$ recovery $\propto D^{1.04}$, log-log slope $1.04$, with zero block censoring, so
the slope is genuine estimation cost), while the recognition cost of $W_{cc}$ and $B$ stays bounded and
dimension-independent. The decoupling advantage $M_0-W_{cc}$ grows from $+193$ to $+4945$ samples over
$D{=}16\!\to\!256$, and $W_{cc}$ never beats the real conjugate-Bayesian $B$ (a bounded gap, insensitive to
$B$'s Dirichlet concentration on the tested cell). A label-shuffle control confirms the advantage is
recognition-driven rather than a generic warm start or a label leak: permuting which regime each block is
scored against collapses $W_{cc}$'s recurrence advantage by $75$--$2947\times$ (Fig.~\ref{fig:realdna}B). Two
scope points keep the reading honest. The match is to the $\Theta$-\emph{scaling}, not the absolute constants
(the prefactor is set by the recovery-time definition and the regime geometry). And because the realized
separation $\Delta$ co-varies with $k$-mer order here, the recognition cost is more precisely
\emph{$\Delta$-tracking} (it follows $1/\Delta^2$) than strictly flat in $D$, which is itself the decoupling
claim, that recognition is governed by separation, not dimension.

\begin{figure}[t]
\centering
\includegraphics[width=\textwidth]{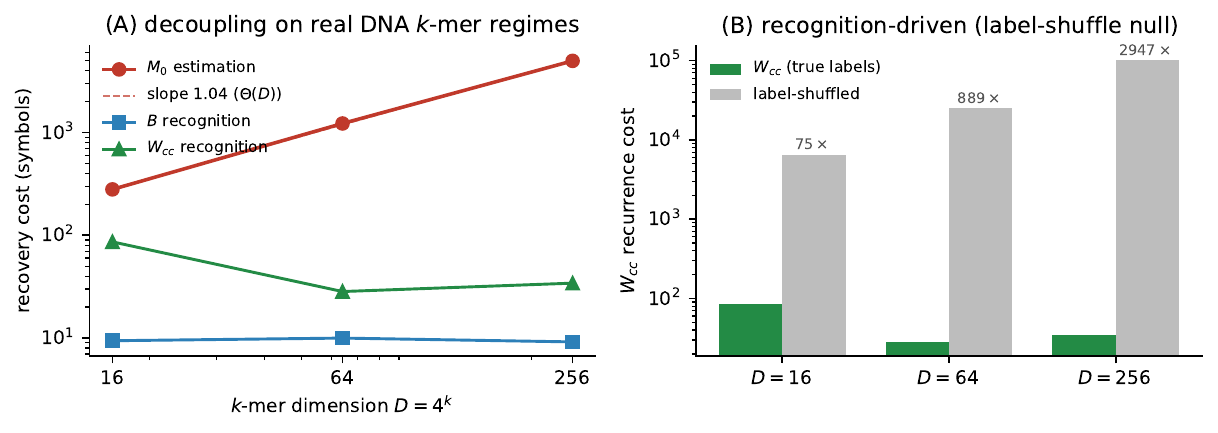}
\caption{\textbf{The decoupling on real regime distributions} ($K{=}4$ regimes $=$ empirical $k$-mer laws of
four genomes, GC $19$--$72\%$; i.i.d.\ block draws; supplied segmentation; five seeds $\times$ two recurrence
densities). \textbf{(A)} Memoryless estimation cost grows $\propto D^{1.04}$ ($\Theta(D/\varepsilon^2)$, zero
censoring) while $W_{cc}$ and $B$ recognition cost stays dimension-independent. \textbf{(B)} A label-shuffle
null (the same $W_{cc}$ predictions scored against a permuted regime schedule) collapses the recurrence
advantage $75$--$2947\times$, recognition-driven, not a generic warm start or a label leak.}
\label{fig:realdna}
\end{figure}

\subsection{The tie holds across the Bayesian-baseline class}\label{sec:withinclass}
Proposition~\ref{prop:bocpd-tie} states the tie against a \emph{class} of fair Bayesian learners, but the
controlled suite above instantiates $B$ as a single conjugate-Dirichlet library learner. We therefore
re-ran the suite with three distinct instantiations of the class: conjugate-Dirichlet model-averaging, a
run-length changepoint posterior (BOCPD~\cite{adams2007bocpd} with per-regime Dirichlet--categorical
components), and a Dirichlet-process nonparametric prior. The count-calibrated learner ties each on
per-recurrence recovery (the recurrence gap to all three is the same flat, never-negative overshoot), and
the three members agree with one another to within a sub-sample offset: the run-length learner coincides
with the conjugate library learner to $0.01$ samples, and the Dirichlet-process variant is $0.37$ samples
faster on average, below the one-sample recovery granularity, and a monotone function of the concentration
prior (at concentration $\approx\!1.5$, matching the conjugate library's pseudo-count, the offset vanishes
exactly; the residual at concentration $1.0$ is a prior calibration of no practical significance). A BOCPD
that uses the supplied block length as its dwell prior coincides with the conjugate library learner at every
tested dimension; an agnostic hazard that discards the supplied boundaries over-hedges and is unstable at
high dimension (a known sensitivity of run-length filtering, outside our supplied-segmentation scope). The
tie of \S\ref{sec:paradigm} is thus a property of the class, not of one instantiation.

\section{Related work and differentiation}\label{sec:related}
The nearest architectural neighbor is HiCL~\cite{hicl2025}, and the surface similarities are real. HiCL and
the broader gradient-free continual-learning family share five recurring features: sparse intermediate codes,
prototype-like memories updated online, cosine-similarity routing, a dual fast/slow memory organization
motivated by hippocampal anatomy~\cite{mcclelland1995complementary,kumaran2016learning}, and prioritized
replay of stored experience. The learner analyzed in this paper is a deliberately thin slice of that family.
It shares the online prototype memory and the recognize-then-reuse skeleton, but not the remaining three:
recognition here is a hard sequential test (Lemma~\ref{lem:rec}), not cosine soft-routing; there is a single
library, not a dual fast/slow store; and novelty is handled by identity-clean append, not replay. We name
the full overlap rather than only our slice because the family-level resemblance to HiCL is genuine, which is
what makes HiCL the right comparison point. The differences below are about mechanism, not disguise.

The load-bearing differences are equally clear. HiCL is trained end-to-end by backpropagation, whereas the
present paper studies a learner with no end-to-end backpropagation at all. HiCL stabilizes via
regularization and replay inside a fixed learned architecture, whereas our stability mechanism is identity-
clean append with no edits to committed entries. HiCL is evaluated on sequential discrete-task continual
learning, whereas our object is a continuous recurring-regime stream and the quantity of interest is
amortized post-change recovery cost. And HiCL's routing is differentiable soft cosine gating, whereas our
retrieval step is a hard recognition decision analyzed as sequential hypothesis testing. This is a scoped
mechanism comparison, not a claim that one system outruns the other on a shared benchmark.

\begin{center}
\footnotesize
\setlength{\tabcolsep}{3pt}
\begin{tabular}{lcccc}
\hline
Axis & HiCL & McCulloch (2025) & EqProp & This work \\
\hline
\shortstack[l]{Training\\signal} & \shortstack[c]{end-to-end\\gradient} & \shortstack[c]{gradient-approximating,\\supervised} & \shortstack[c]{supervised\\energy-based gradient} & \shortstack[c]{no gradient,\\no target} \\
\shortstack[l]{Boundary /\\recovery object} & \shortstack[c]{task-sequential\\CL} & \shortstack[c]{recovery without\\supplied boundaries} & \shortstack[c]{not a recovery\\framework} & \shortstack[c]{supplied segmentation;\\recovery cost} \\
\shortstack[l]{Capacity\\model} & \shortstack[c]{fixed\\prototypes} & \shortstack[c]{fixed\\network} & \shortstack[c]{fixed\\network} & \shortstack[c]{append grows\\on novelty} \\
\hline
\end{tabular}
\end{center}

McCulloch's recent preprint \cite{mcculloch2025hierarchical} is the other mandatory deconfliction point
because it is also framed around recovery. The difference in object is substantial. His mechanism is
gradient-approximating local credit assignment with a supervised prediction-error signal; ours computes no
gradient and optimizes no target label at all. His paper studies recovery without supplied boundaries and
presents that behavior as part of a constructive free-energy-style existence argument, whereas our claim is
a condition-gated cost-separation theorem and the no-boundary case is explicitly left outside the main
positive claim, with an impossibility result at the packing wall in \S\ref{sec:threats}. His results are
demonstrated at toy scale; ours are cost-theoretic and benchmark-agnostic. The works are therefore
different in mechanism, supervision, scope, and deliverable rather than competing implementations of the
same idea. EqProp \cite{scellier2017equilibrium} is a third useful contrast point: it is a biologically
motivated route to supervised gradient computation, whereas this paper is about unsupervised recognition and
warm-start reuse.

A concurrent formal treatment of recovery exists for gradient-based test-time adaptation: Zhou et
al.~\cite{zhou2026ttalearnability} define an $(\varepsilon,\delta)$-recovery complexity, the first time
after a shift beyond which a proxy-gradient adapter's probability of exceeding excess risk $\varepsilon$
stays below $\delta$, prove upper and lower bounds for it in a stochastic proxy-gradient oracle model,
and transfer the per-shift quantity to a stream-level $(\varepsilon,\rho)$-learnability guarantee through a
shift-count bound. Their analysis does not distinguish first occurrences from returns of previously
encountered distributions and contains no memory, library, or warm-start component, whereas
Theorem~\ref{thm:sep} studies exactly that distinction under a categorical recurring-regime model with
supplied segmentation. Relating the two guarantees would require a translation between their
proxy-gradient excess-risk setting and our TV-recovery setting, which neither paper supplies.

Relative to classical theory, the present contribution is intentionally narrow. Sequential testing supplies
the recognition primitive \cite{lorden1971procedures,moustakides1998quickest}; prediction-with-experts
clarifies where the $\log K$ dependence belongs \cite{cesa-bianchi-lugosi-2006}; switching- and
tracking-regret with memory is the online-learning home of the amortized structure itself, where tracking
the best of a small set of experts by mixing past posteriors pays a per-switch $\log K$ on top of a
per-expert learning cost \cite{herbster1998tracking,bousquet2002tracking}, with sharper long-term-memory
guarantees in later work \cite{zheng2019equipping,robinson2021tracking}; the same pay-once-then-switch
structure appears information-theoretically as universal coding for sources with repeating statistics
\cite{shamircostello2004repeating}, and the per-first-occurrence
learning cost of a parametric source is the $\tfrac{D-1}{2}\log L$ universal-coding redundancy of
Xie--Barron \cite{xiebarron2000}; BOCPD and related online
changepoint methods provide a fair spawn-capable Bayesian baseline \cite{adams2007bocpd}; Bayesian
nonparametric mixtures such as DPMMs provide another route to state growth under novelty
\cite{rasmussen2000infinite}; and meta-learning theory studies when prior tasks reduce the sample complexity
of adaptation to new ones~\cite{baxter2000model,maurer2016benefit}. Our recurrence term is exactly this
meta-learning effect specialized to recurring categorical regimes: the $K$ first occurrences pay the full
per-task estimation once, and every recurrence is adaptation at the dimension-free $\Theta(\log K/\Delta^2)$
rate, the prior-task sample-complexity reduction made exact, and capped by the packing-limited capacity of
\S\ref{sec:geom}. The estimation--recognition split also has a clean property-testing reading: \emph{learning}
or \emph{testing} an \emph{unknown} categorical law is domain-size-dependent: identity and closeness testing
are sublinear in the support but still grow with it, even in the sequential
setting~\cite{canonne2020survey,fawzi2022sequential}, whereas \emph{recognizing} which of $K$
already-estimated, $\Delta$-separated \emph{known} candidates is active is a multihypothesis discrimination
that is domain-size-\emph{independent} ($\Theta(\log K/\Delta^2)$, Lemma~\ref{lem:rec}). The decoupling we
exploit is precisely this known-vs-unknown gap, and the recognition primitive itself is classical.
Our novelty claim is not that any one ingredient is new in isolation. It is that
these pieces combine into a scoped theorem about recovery on recurring-regime streams, with a born-immune
memoryless baseline, a dimension-independent recognition term, explicit random-model and packing walls, and
a gradient-free append-only realization that ties rather than exceeds the fair Bayesian alternative.

\section{Threats and scope}\label{sec:threats}
This section states the boundaries in the strongest form we think is defensible. The point is not to soften
the main theorem, but to separate the condition-gated positive result from the regimes where the mechanism
provably stops helping.

\begin{theorem}[Autonomous segmentation is impossible at the packing wall]\label{thm:auto}
In the autonomous setting (no oracle segmentation) the learner must itself detect block boundaries by
sequential change detection. Any detector with mean time to false alarm $\ge 1/\alpha$ incurs worst-case
detection delay $\ge(1-o(1))\,|\log\alpha|/I$, where $I$ is the information (Kullback--Leibler) between the
pre- and post-change laws (Lorden/Lai)~\cite{lorden1971procedures,lai1998information}, and $I\asymp\Delta^2$ (Pinsker and reverse-Pinsker~\cite{sason2015reverse}, the latter
under interior support $p_{\min}>0$). In the \emph{separated} corner ($\Delta$ bounded below) this delay is of
the same $\Theta(1/\Delta^2)$ order as $\nrec$, so autonomy preserves the condition-gated \emph{rate} at a
constant factor (it is not free). At the \emph{packing wall} (Lemma~\ref{lem:pack}: $\Delta\to0$ as regimes
pack at fixed $D$, hence $I\to0$) the delay diverges: \emph{no} detector (fixed, adaptive, or Bayesian) escapes the false-alarm/delay frontier, so the condition-gated rate is unattainable autonomously. The
impossibility is paradigm-neutral (a Bayesian spawn-capable baseline fails it too) and is scoped to the
overlap/packing regime, not to the separated corner where the main results hold.
\end{theorem}
\begin{remark}[Why a tuned fixed threshold suffices in the separated corner]\label{rem:frontier}
The frontier of Theorem~\ref{thm:auto} is detector-class-agnostic. In the separated corner it sits well
below the operating point, so a single well-placed threshold attains both low false alarm and bounded
delay and adaptive machinery buys nothing; near the wall the frontier rises above the operating point, so
no rule escapes. This is consistent with a separate line of internal negative results in which every tested
autonomous spawn rule (statistical-test, sequential-Monte-Carlo, and an adaptive threshold calibrator) is
dominated by a tuned fixed threshold in the separable regime, while none succeeds near overlap.
\end{remark}
\begin{remark}[Remaining scope boundaries]\label{rem:scope}
The separation also requires $\Delta$ bounded below (it collapses under regime overlap), and is an abstract
architecture result: the cost structure must be realized by an added retrieval/library module, not by a
generic engine; a non-gradient persistent-memory substrate of the kind this requires (triple-loop
consolidation that maintains context-specific memory without gradients) is developed in a related line of work~\cite{lou2026frontiers}. The empirical companion is corroboration on a single real stream, not a
benchmark claim.
\end{remark}
\begin{remark}[(A6) is the worst-case corner of a known amortization frontier]\label{rem:a6frontier}
The cold-total lower bound $\Omega(KD/\varepsilon^2)$ (Prop.~\ref{prop:bocpd-tie}) is stated under (A6), no
cross-regime amortization, and (A6) is the worst-case corner of a classical phenomenon rather than a hidden
restriction. It is tight exactly when the $K$ regime laws share no exploitable structure. When they instead
lie near a common low-dimensional structure inside the simplex, samples drawn under one regime inform the
others, and the cold total is governed by the family's metric entropy rather than by $KD$: a learner can
estimate the shared structure once and place each regime within it, so the cold total falls below
$KD/\varepsilon^2$ at the multitask/representation-learning minimax rates
\cite{yangbarron1999,tripuraneni2020theory,maurer2016benefit}. We state the separation at the (A6)
boundary, where the cold total is genuinely $\Omega(KD/\varepsilon^2)$, and do not claim the structured
improvement, which needs a shared-structure realizability condition and a different lower-bound instrument
and is outside our scope. This scopes where the \emph{joint} statement is tight; it does not weaken the
recurrence term, which is two-sided (Theorem~\ref{thm:rec-minimax}) and never invokes (A6).
\end{remark}

\bibliographystyle{unsrt}
\bibliography{references}

\appendix
\section{Proofs}

\begin{proof}[Proof of Lemma~\ref{lem:est}]
Write $\widehat\theta_n$ for the empirical histogram from $n$ i.i.d.\ draws from $\theta\in\simplex$.
Since $\TV(\widehat\theta_n,\theta)=\tfrac12\|\widehat\theta_n-\theta\|_1$, the Weissman $\ell_1$ tail bound
\cite{weissman2003l1} gives
\[
\Pr_\theta\!\big(\TV(\widehat\theta_n,\theta)>\varepsilon\big)
=
\Pr_\theta\!\big(\|\widehat\theta_n-\theta\|_1>2\varepsilon\big)
\le
(2^D-2)e^{-2n\varepsilon^2}.
\]
Thus $n\ge C(D+\log(1/\delta))/\varepsilon^2$ suffices for $\Pr(\TV(\widehat\theta_n,\theta)\le\varepsilon)\ge1-\delta$.

For the lower bound, this is the standard multinomial minimax rate in $\ell_1$ loss. In particular, over the
simplex interior one has
\[
\inf_{\widetilde\theta}\sup_{\theta\in\simplex}\mathbb E_\theta\|\widetilde\theta-\theta\|_1
\ge c\sqrt{D/n},
\]
by Assouad/Fano/Le Cam constructions for multinomial families; see, e.g., \cite[Chapter~2]{tsybakov2009nonparametric}.
Hence
\[
\inf_{\widetilde\theta}\sup_\theta \mathbb E_\theta \TV(\widetilde\theta,\theta)\ge c'\sqrt{D/n}.
\]
If some estimator achieved $\TV\le\varepsilon$ with probability at least $2/3$ uniformly over $\theta$, then
\[
\sup_\theta \mathbb E_\theta \TV(\widetilde\theta,\theta)
\le \varepsilon+\tfrac13,
\]
and after the standard scaling of the interior hypercube construction to target accuracy $\varepsilon$, this forces
$n\ge c''D/\varepsilon^2$. The dependence on $\log(1/\delta)$ is standard and orthogonal to the $D/\varepsilon^2$
rate. Therefore $\nest(D,\varepsilon)=\Theta(D/\varepsilon^2)$ up to confidence factors.
\end{proof}

\begin{proof}[Proof of Lemma~\ref{lem:rec}]
Assume first an oracle library of simple hypotheses $\{q_1,\dots,q_K\}$ with the active law equal to some
$q_{j^*}$ and $\min_{j\neq j^*}\TV(q_{j^*},q_j)\ge\Delta$. We also make explicit the load-bearing interior-support
assumption used by the text: each $q_j$ has strictly positive coordinates, bounded below by some $p_{\min}>0$.
This guarantees finite log-likelihood increments and allows standard Wald-style stopping arguments.

For each competitor $j\neq j^*$ define the log-likelihood ratio random walk
\[
S_t^{(j)}:=\sum_{s=1}^t \log\frac{q_{j^*}(X_s)}{q_j(X_s)}.
\]
Under $P_{j^*}$, its drift is
\[
\mu_j:=\mathbb E_{j^*}[S_1^{(j)}]=D_{\mathrm{KL}}(q_{j^*}\|q_j).
\]
By Pinsker,
\[
\mu_j\ge 2\,\TV(q_{j^*},q_j)^2\ge 2\Delta^2,
\]
and equivalently the Chernoff information is bounded below by a constant multiple of $\Delta^2$
\cite{chernoff1952measure,bretagnolle1979estimation}. Consider the multihypothesis SPRT/GLR rule that stops at
\[
\tau:=\inf\Big\{t:\exists i\ \text{s.t.}\ S_t^{(i,j)}\ge A\ \text{for all }j\neq i\Big\},
\qquad
A:=\log\frac{K-1}{\delta}+c_0,
\]
and outputs that index $i$; here $S_t^{(i,j)}$ is the pairwise log-likelihood ratio between $q_i$ and $q_j$.
Under $P_{j^*}$, each competing walk has positive drift at least $2\Delta^2$, so by Wald's identity and the
standard multihypothesis SPRT analysis \cite{wald1947sequential,dragalin1999multihypothesis},
\[
\mathbb E_{j^*}\tau \le \frac{A+O(1)}{\min_{j\neq j^*}\mu_j}
\le C\frac{\log(K/\delta)}{\Delta^2}.
\]
The error bound is the usual one: for each fixed competitor $j\neq j^*$, the probability that the wrong
pairwise test crosses first is at most $e^{-A}$, hence
\[
P_{j^*}(\widehat J\neq j^*)\le \sum_{j\neq j^*} e^{-A}\le \delta
\]
by a union bound over the $K-1$ competitors. This proves
\[
\nrec(K,\Delta,\delta)=O\!\left(\frac{\log(K/\delta)}{\Delta^2}\right).
\]
\end{proof}

\begin{proof}[Proof of Theorem~\ref{thm:rec-minimax}]
We state the minimax class explicitly. Let $\mathfrak S(K,\Delta,\delta)$ be the class of all sequential tests
$(\tau,\widehat J)$ such that:
\begin{enumerate}
\item $\tau$ is a stopping time for the observation filtration;
\item $\widehat J\in\{1,\dots,K\}$ is $\mathcal F_\tau$-measurable;
\item the library $\{q_1,\dots,q_K\}$ lies in the simplex interior and satisfies
$\min_{i\neq j}\TV(q_i,q_j)\ge\Delta$;
\item the maximal error is controlled uniformly:
$\sup_j P_j(\widehat J\neq j)\le\delta$, for some fixed $\delta<1/2$.
\end{enumerate}
The theorem concerns
\[
\inf_{(\tau,\widehat J)\in\mathfrak S(K,\Delta,\delta)}\ \sup_{1\le j\le K}\mathbb E_j\tau.
\]

\emph{Step 1: the $1/\Delta^2$ factor.}
Take two Bernoulli laws
\[
q_0=(\tfrac12+\Delta,\tfrac12-\Delta),\qquad q_1=(\tfrac12-\Delta,\tfrac12+\Delta),
\]
which satisfy $\TV(q_0,q_1)=\Delta$ for $\Delta<1/2$. For any sequential test with
$P_0(\widehat J\neq0)\vee P_1(\widehat J\neq1)\le\delta$, the standard change-of-measure inequality
(equivalently Wald's likelihood identity for simple-vs-simple testing) yields
\[
\mathbb E_0\tau\; D_{\mathrm{KL}}(q_0\|q_1)\ge d(1-\delta,\delta),
\]
where $d(a,b)=a\log(a/b)+(1-a)\log((1-a)/(1-b))>0$ depends only on $\delta$ \cite{wald1947sequential}.
Since $D_{\mathrm{KL}}(q_0\|q_1)=O(\Delta^2)$ as $\Delta\downarrow0$, one gets
\[
\sup_{j\in\{0,1\}}\mathbb E_j\tau \ge c_\delta\,\Delta^{-2}.
\]
Thus no sequential recognizer can beat the $\Delta^{-2}$ dependence.

\emph{Step 2: the $\log K$ factor as a sequential lower bound.}
It remains to build a $K$-point family with pairwise $\TV\ge\Delta$ but pairwise KL only $O(\Delta^2)$.
Let $m$ be such that $K\le e^{cm}$ for the Gilbert--Varshamov constant $c>0$, and choose a binary code
$\mathcal V\subset\{\pm1\}^m$ of size at least $K$ and pairwise Hamming distance at least $m/4$
\cite{gilbert1952comparison,varshamov1957estimate}. For $v\in\mathcal V$, define a categorical law on $2m$
symbols by
\[
q_v(2i-1)=\frac{1+\eta v_i}{2m},\qquad
q_v(2i)=\frac{1-\eta v_i}{2m},
\qquad i=1,\dots,m,
\]
with $0<\eta<1/4$ chosen below. All coordinates are positive, so the family lies in the simplex interior.
If $v,w$ differ on $h(v,w)$ coordinates, then
\[
\TV(q_v,q_w)=\frac{h(v,w)\eta}{m}\ge \frac{\eta}{4}.
\]
Choosing $\eta=4\Delta$ (for $\Delta<1/16$; larger constant $\Delta$ only changes constants) gives pairwise
separation at least $\Delta$.

For the KL divergence, a direct calculation shows that each differing coordinate pair contributes
$O(\eta^2/m)$, hence
\[
D_{\mathrm{KL}}(q_v\|q_w) \le C\eta^2 \le C'\Delta^2
\qquad \text{for all }v\neq w.
\]
Now put the uniform prior on the $K$ selected hypotheses and let
$Y=(\tau,X_1,\dots,X_\tau)$ be the stopped transcript. Fano's inequality gives
\[
I(V;Y)\ge (1-\delta)\log K-\log 2.
\]
On the other hand, by convexity of mutual information and the stopped-log-likelihood identity,
\[
I(V;Y)\le \frac1{K^2}\sum_{v,w} D(P_v^Y\|P_w^Y)
= \frac1{K^2}\sum_{v,w}\mathbb E_v\tau\; D_{\mathrm{KL}}(q_v\|q_w)
\le C'\Delta^2 \sup_v \mathbb E_v\tau.
\]
Combining the last two displays yields
\[
\sup_v\mathbb E_v\tau \ge c_{\delta}\frac{\log K}{\Delta^2}.
\]
This is exactly the desired \emph{sequential expected-stopping-time} lower bound; no fixed-sample reduction is
used. (The construction needs ambient dimension $D\gtrsim\log K$, so that a rate-$c$ Gilbert--Varshamov code
of length $m=\Theta(\log K)$ embeds in $D=2m$ symbols; for smaller $D$ the $\log K$ factor saturates at the
log-cardinality of $D$-distinguishable regimes, i.e.\ the packing wall of Lemma~\ref{lem:pack}.)

\emph{Conclusion.}
Lemma~\ref{lem:rec} gives the matching upper bound up to the usual $\log(1/\delta)$ factor, so
\[
\nrec=\Theta(\log K/\Delta^2)
\]
on interior-supported separated libraries. The last sentence of the theorem follows immediately by dividing
this by Lemma~\ref{lem:est}'s $\nest=\Theta(D/\varepsilon^2)$.
\end{proof}

\begin{proof}[Proof of Lemma~\ref{lem:noisy}]
Let $\theta_{r(j)}$ denote the true regime behind library entry $q_j$, and assume
$\TV(q_j,\theta_{r(j)})\le \elib$ for all $j$. If the active regime is $j^*$, then for any competitor $j\neq j^*$,
\[
\TV(q_{j^*},q_j)
\ge
\TV(\theta_{r(j^*)},\theta_{r(j)})-\TV(q_{j^*},\theta_{r(j^*)})-\TV(q_j,\theta_{r(j)})
\ge \Delta-2\elib
\]
by the triangle inequality. Therefore the noisy library is still separated, but only by the effective margin
$\Delta_{\mathrm{eff}}:=\Delta-2\elib$.

If $\elib<\Delta/2$, then $\Delta_{\mathrm{eff}}>0$, and Lemma~\ref{lem:rec} applied to the estimated
prototypes gives
\[
\nrec=O\!\left(\frac{\log K}{(\Delta-2\elib)^2}\right).
\]
If $\elib\ge\Delta/2$, the lower bound above becomes nonpositive, and the statement correctly warns that the
library may be non-identifiable at resolution $\elib$.

For the residual recovery claim, after the first cold occurrence of regime $j$ the learner has paid the
estimation cost and stored a prototype satisfying $\TV(q_j,\theta_{r(j)})\le\varepsilon$. On any later
recurrence, once recognition returns the correct index $j$, the warm start is exactly $q_j$, hence already
within tolerance:
\[
\TV(q_j,\theta_{r(j)})\le\varepsilon.
\]
So the post-recognition residual recovery time is $0$. When $\varepsilon/\Delta$ is small, the clean oracle
bound of Lemma~\ref{lem:rec} is therefore perturbed only through the replacement
$\Delta\mapsto \Delta-2\varepsilon$.
\end{proof}

\begin{proof}[Proof of Theorem~\ref{thm:sep}]
The decomposition is by block type.

On the first occurrence of each of the $K$ regimes, the warm-start learner has no correct prototype yet. It
must estimate that regime to tolerance $\varepsilon$ and append the resulting prototype. By
Lemma~\ref{lem:est}, this costs $\nest=\Theta(D/\varepsilon^2)$ per new regime, hence $K\nest$ in total.

On each of the remaining $R-K$ recurrent blocks, the learner only needs to recognize which stored regime has
returned and then warm-start from that prototype. Here one load-bearing assumption must be surfaced
explicitly: because the library entries are estimates rather than oracle laws, we need the self-consistency
condition from Lemma~\ref{lem:noisy},
\[
\varepsilon\le \elib<\Delta/2,
\]
in order for the estimated library still to be separated. Under that condition, each recurrent block costs
\[
\nrec=O\!\left(\frac{\log K}{(\Delta-2\elib)^2}\right)
=O\!\left(\frac{\log K}{\Delta^2}\right)
\]
for fixed $\Delta$ bounded below and $\elib\le\varepsilon=o(\Delta)$, and after recognition the residual is
$0$ by Lemma~\ref{lem:noisy}. Summing over the $R-K$ recurrent blocks gives
\[
\Cost(W)\le K\,\nest+(R-K)\,\nrec.
\]
Substituting Lemma~\ref{lem:est} and Lemma~\ref{lem:rec} yields
\[
\Cost(W)\le K\,\Theta(D/\varepsilon^2)+(R-K)\,O(\log K/\Delta^2).
\]

For $M_0$, Corollary~\ref{cor:m0} gives
\[
\Cost(M_0)=\Theta(RD/\varepsilon^2).
\]
Therefore
\[
\Cost(M_0)-\Cost(W)
\ge
(R-K)\Big(\Theta(D/\varepsilon^2)-O(\log K/\Delta^2)\Big).
\]
If $D\Delta^2/\varepsilon^2\gg \log K$, the first term dominates and the gap is
$(R-K)\Theta(D/\varepsilon^2)(1-o(1))$. If $D=O(1)$, or $\Delta\to0$, or the estimated-library condition
$\varepsilon<\Delta/2$ fails, the recognition advantage disappears, exactly as stated.
\end{proof}

\begin{proof}[Proof of Proposition~\ref{prop:born}]
By Definition~\ref{def:classes}, an $M_0$ learner carries no cross-block state. Hence on every block it
starts from reset and must solve the full estimation problem anew. There is no separate identification stage
against a stored library, because no library exists. Operationally, its ``recognition'' problem is identical
to ``estimate the current regime from scratch,'' so
\[
\nrec^{M_0}=\nest.
\]
Substituting this into the decomposition of Theorem~\ref{thm:sep} collapses the putative separation:
replacing every recurrence by a full re-estimation leaves exactly the memoryless cost. Thus the advantage is
identically $0$ for $M_0$, by construction.
\end{proof}

\begin{proof}[Proof of Lemma~\ref{lem:floor}]
Let $\theta,\theta'\stackrel{\mathrm{i.i.d.}}{\sim}\mathrm{Dir}(\alpha\mathbf 1_D)$ and define
\[
F_D(\theta,\theta'):=\TV(\theta,\theta')=\frac12\sum_{i=1}^D |\theta_i-\theta_i'|.
\]
The proof has two ingredients: a constant-order mean and $D^{-1/2}$ concentration.

First, by exchangeability,
\[
\mathbb E F_D=\frac{D}{2}\,\mathbb E|\theta_1-\theta_1'|.
\]
The one-coordinate marginal is Beta$(\alpha,(D-1)\alpha)$. Writing $\theta_1=G/(G+H)$ with
$G\sim\Gamma(\alpha,1)$ and $H\sim\Gamma((D-1)\alpha,1)$, one has $D\theta_1\Rightarrow G/\alpha$ as
$D\to\infty$, and the family is uniformly integrable. Therefore
\[
\mathbb E F_D = c(\alpha)+O(D^{-1/2}),
\qquad
c(\alpha):=\frac12\,\mathbb E\big|Y-Y'\big|,
\]
where $Y,Y'$ are i.i.d.\ copies of the nondegenerate limit law of $D\theta_1$. In particular, the mean is
asymptotically $D$-independent; this is the sense in which the lemma's constant $c(\alpha)$ is used in the
main text.

Second, $F_D$ is a bounded-difference functional of the underlying normalized-gamma representation
\[
\theta_i=\frac{G_i}{\sum_{\ell=1}^D G_\ell},
\qquad
\theta_i'=\frac{G_i'}{\sum_{\ell=1}^D G_\ell'}.
\]
On the high-probability event that both denominators are between $c_1D$ and $c_2D$, changing one gamma
coordinate changes each normalized vector in $\ell_1$ by at most $C(\alpha)/D$, hence changes $F_D$ by at
most $C'(\alpha)/D$. McDiarmid's inequality on that event, together with exponentially small tails for the
complements, yields
\[
\Pr\!\left(|F_D-\mathbb EF_D|>t\right)\le 2e^{-\kappa(\alpha)Dt^2}.
\]
Equivalently, $F_D$ has standard deviation at most $a(\alpha)/\sqrt D$ for some $a(\alpha)>0$.

Now take $K$ i.i.d.\ Dirichlet draws. There are at most $K^2/2$ pairs, so applying the preceding tail bound
to each pair and union-bounding gives that, with probability at least $1-K^{-c'}$,
\[
\Delta(K,D)
\ge
\mathbb EF_D - C a(\alpha)\sqrt{\frac{\log K}{D}}
\ge
c(\alpha)-2a(\alpha)\sqrt{\frac{\log K}{D}},
\]
after absorbing the $O(D^{-1/2})$ bias of $\mathbb E F_D$ into the constants. This is exactly the claimed
separation floor.
\end{proof}

\begin{proof}[Proof of Lemma~\ref{lem:pack}]
Because $\TV(\theta,\theta')=\tfrac12\|\theta-\theta'\|_1$, the problem is an $\ell_1$ packing problem on the
$(D-1)$-dimensional simplex. Let $\delta:=\min_{r\neq r'}\TV(\theta_r,\theta_{r'})$. Then the closed
$\ell_1$ balls of radius $\delta$ around the $K$ points are disjoint after the factor-$2$ identification
between $\TV$ and $\ell_1$.

For the upper bound, intersect these balls with the affine hyperplane $\sum_i x_i=1$. Their volumes are those
of $(D-1)$-dimensional $\ell_1$ cross-polytopes, so
\[
K\,\mathrm{vol}_{D-1}\!\big(B_1^{D-1}(\delta)\big)
\le
\mathrm{vol}_{D-1}(\simplex),
\]
hence
\[
\delta \le c_D\,K^{-1/(D-1)},
\]
with
\[
c_D:=\left(\frac{\mathrm{vol}_{D-1}(\simplex)}
{\mathrm{vol}_{D-1}(B_1^{D-1}(1))}\right)^{1/(D-1)}.
\]
Stirling's formula gives $c_D=\Theta(1)$ and $c_D\to 1/2$, since in $\TV$ distance the simplex is exactly the
$\ell_1$ cross-polytope scaled by the factor $1/2$.

For the converse, use the same binary-pair embedding as in the proof of
Theorem~\ref{thm:rec-minimax}. For any fixed $\Delta_0\in(0,1/8)$ and any $m$, Gilbert--Varshamov gives a
code $\mathcal V\subset\{\pm1\}^m$ of size $e^{\Theta(m)}$ and pairwise Hamming distance $\Theta(m)$
\cite{gilbert1952comparison,varshamov1957estimate}. Mapping $v\in\mathcal V$ to
\[
q_v(2i-1)=\frac{1+\eta v_i}{2m},\qquad
q_v(2i)=\frac{1-\eta v_i}{2m},
\]
with $\eta$ chosen so that the resulting pairwise TV distance equals at least $\Delta_0$, produces
$e^{\Theta(m)}$ points in $\Delta^{2m-1}$ at fixed positive pairwise separation. Writing $D=2m$ gives
$e^{\Theta(D)}$ points at fixed separation in dimension $D$. This is the claimed exponential packing
capacity. The same conclusion is equivalent to standard simplex covering-number bounds, including the
Kabatiansky--Levenshtein viewpoint \cite{kabatiansky1978bounds}.
\end{proof}

\begin{proof}[Proof of Proposition~\ref{prop:bocpd-tie}]
To keep the proof within what the cited tools support, we prove the proposition for one precise baseline:
\[
B_{\mathrm{lib}}:=\text{a Bayesian online learner with a \emph{fixed} library of $K$ Dirichlet--categorical components and supplied segmentation.}
\]
The broader BOCPD/HMM/DPMM wording in the main text should be read as informal motivation, not as an
additional theorem proved here.

Within a block, $B_{\mathrm{lib}}$ updates the posterior over the $K$ component labels by Bayes' rule and
updates the active component's Dirichlet posterior conjugately. On the first occurrence of regime $r$, the
learner must infer that regime's categorical law from its block samples; by the same concentration bound as
Lemma~\ref{lem:est}, the Dirichlet posterior mean reaches $\TV\le\varepsilon$ after
$\Theta(D/\varepsilon^2)$ samples. Summing over the $K$ first occurrences gives the cold cost
$O(KD/\varepsilon^2)$.

On a recurrence of an already-instantiated regime $j^*$, the library contains the correct component. For any
competitor $j\neq j^*$, the posterior log-odds increment is exactly the one-sample log-likelihood ratio
\[
Z_t^{(j)}=\log\frac{q_{j^*}(X_t)}{q_j(X_t)},
\]
so under the true component its drift is
$\mathbb E_{j^*} Z_t^{(j)}=D_{\mathrm{KL}}(q_{j^*}\|q_j)\ge 2\Delta^2$ by Pinsker. Therefore the posterior
selects the correct component after
\[
O(\log K/\Delta^2)
\]
samples by the same Wald/multihypothesis argument used in Lemma~\ref{lem:rec}. Once the posterior mass has
concentrated on the correct component, the predictive distribution is the corresponding conjugate posterior
mean, which is already within the estimation tolerance inherited from the component's first occurrence; the
residual recovery is therefore negligible at the order tracked here. Summing over the $R-K$ recurrences gives
\[
\Cost(B_{\mathrm{lib}})
=
O\!\big(KD/\varepsilon^2+(R-K)\log K/\Delta^2\big).
\]
This matches the upper bound proved for $\Cost(W)$ in Theorem~\ref{thm:sep}. No stronger statement is needed
for the paper's ``ties, not beats'' conclusion.
\end{proof}

\begin{proof}[Proof of Theorem~\ref{thm:cc}]
We isolate the per-recurrence comparison, since the theorem is an upper bound on the recurrence overshoot.
Fix a recurrent block from true regime $j^*$, and let $q_j$ denote the stored prototype means with counts
$n_j$. The count-calibrated rule compares the hard count log-likelihood ratios
\[
S_t^{(j)}:=\sum_{s=1}^t \log\frac{q_{j^*}(X_s)}{q_j(X_s)},\qquad j\neq j^*,
\]
and commits once every competitor has been beaten by the chosen threshold. The Bayesian baseline $B$ uses the
same evidence, but retains soft posterior weights and posterior variances. The difference between the two
stopping rules is therefore a pure overshoot phenomenon: $W_{cc}$ waits until the \emph{hard} threshold is
crossed, whereas $B$ can exploit the same evidence fractionally through posterior weighting.

Under the interior-support assumption $q_j(x)\ge p_{\min}>0$, each one-step increment satisfies
\[
Z_s^{(j)}:=\log\frac{q_{j^*}(X_s)}{q_j(X_s)} \le \log(1/p_{\min})=:B_{\max},
\]
because the numerator is at most $1$ and the denominator at least $p_{\min}$. Its drift is
\[
\mu_j:=\mathbb E_{j^*} Z_s^{(j)} = D_{\mathrm{KL}}(q_{j^*}\|q_j)\ge 2\Delta^2,
\]
again by Pinsker. Wald's overshoot bound for positive-drift random walks---a bounded-increment renewal
overshoot is at most the maximal increment~\cite{wald1947sequential,lorden1971procedures}---then gives
\[
\mathbb E_{j^*}\big[\tau_{cc}-\tau_B\big]_+ \le \frac{B_{\max}+O(1)}{\min_{j\neq j^*}\mu_j}
\le
C\frac{\log(1/p_{\min})}{\Delta^2}.
\]
This is the theorem's $\xi$ term. The conclusion is therefore an \emph{upper} bound,
\[
\Cost(W_{cc})\le \Cost(B)+\xi,
\qquad
\xi=O\!\left(\frac{\log(1/p_{\min})}{\Delta^2}\right),
\]
not an equality and not a lower bound.

The rest of the theorem is bookkeeping. Because the first occurrence of each regime lasts at least
$\nest=\Theta(D/\varepsilon^2)$ samples by assumption $L\ge \nest$, the stored count $n_j$ after that first
occurrence is already large enough for the prototype mean to satisfy $\TV(q_j,\theta_j)\le\varepsilon$.
Hence on every later recurrence, once the component is recognized, the predictive starts within tolerance
with no free pseudo-count choice: the warm-start precision is the regime's own count $n_j$. Finally, the
specialization
\[
\xi=O(\log D/\Delta^2)
\]
requires the additional assumption $p_{\min}=\Theta(1/D)$; without it, the correct statement is only the
general $O(\log(1/p_{\min})/\Delta^2)$ bound.
\end{proof}

\begin{proof}[Proof of Theorem~\ref{thm:auto}]
We first make the autonomous class precise. A procedure is a stopping time $\tau$ adapted to the observation
filtration, together with a post-change decision rule, and it is required to satisfy a false-alarm guarantee
\[
\mathbb E_\infty \tau \ge 1/\alpha,
\]
where $P_\infty$ is the no-change law. Its performance criterion is worst-case detection delay in Lorden's
sense,
\[
\mathcal D(\tau):=\sup_{\nu\ge1}\operatorname*{ess\,sup}
\mathbb E_\nu\!\left[(\tau-\nu+1)_+\mid \mathcal F_{\nu-1}\right].
\]
This is the standard minimax class for sequential change detection.

Now fix two categorical laws $p,q$ representing adjacent regimes. Lorden's and Lai's lower bounds state that
for any such procedure,
\[
\mathcal D(\tau)\ge (1-o(1))\frac{|\log\alpha|}{I(p,q)}
\qquad\text{as }\alpha\downarrow0,
\]
where $I(p,q)$ is the relevant Kullback--Leibler information number; for the simple i.i.d.\ one-change model,
$I(p,q)=D_{\mathrm{KL}}(q\|p)$ \cite{lorden1971procedures,lai1998information}. Thus any detector with long
mean time to false alarm must pay delay inversely proportional to the pre/post divergence.

To connect this with the paper's separation parameter $\Delta=\TV(p,q)$, use Pinsker:
\[
I(p,q)\ge 2\Delta^2.
\]
Under the additional interior-support condition $p_i,q_i\ge p_{\min}>0$, reverse-Pinsker inequalities imply
\[
I(p,q)\le C(p_{\min})\,\Delta^2
\]
for a finite constant $C(p_{\min})$ depending only on the support floor \cite{sason2015reverse}. Hence in the
separated interior class,
\[
I(p,q)\asymp \Delta^2,
\]
so the unavoidable autonomous detection delay is of order $|\log\alpha|/\Delta^2$. This matches the same
$1/\Delta^2$ scale as recognition, up to the false-alarm factor.

At the packing wall of Lemma~\ref{lem:pack}, however, fixed $D$ and growing library size force
$\Delta\to0$. Therefore $I(p,q)\to0$ as well, and the lower bound becomes
\[
\mathcal D(\tau)\ge (1-o(1))\frac{|\log\alpha|}{I(p,q)}\to\infty.
\]
So no autonomous detector can simultaneously maintain the false-alarm guarantee and uniformly bounded
detection delay in that regime. This is the precise sense in which autonomous recovery is impossible at the
packing wall. The argument is detector-class-agnostic, so it applies equally to fixed-threshold, adaptive,
and Bayesian procedures.
\end{proof}

\paragraph{Proof-status summary.}
The proofs of Lemmas~\ref{lem:est}, \ref{lem:rec}, \ref{lem:noisy}, \ref{lem:floor}, \ref{lem:pack} and
Theorems~\ref{thm:rec-minimax}, \ref{thm:sep}, \ref{thm:cc}, \ref{thm:auto}, as well as
Propositions~\ref{prop:born}, \ref{prop:bocpd-tie}, are complete at the paper's claimed level: each reduces
the statement to standard cited tools plus explicit reductions. The only intentionally sketch-level pieces
are the exact finite-$D$ constants in Lemma~\ref{lem:floor} and the exact cross-polytope volume constant in
Lemma~\ref{lem:pack}; the paper uses only their order and asymptotic form, so carrying those constants
further would be bookkeeping rather than new argument. (The impossibility of Theorem~\ref{thm:auto} is
established within the stated Lorden minimax change-detection class, as its statement scopes.)

\end{document}